\documentclass[letterpaper]{article}

\usepackage{aaai2026}
\usepackage{times}
\usepackage{helvet}
\usepackage{courier}
\usepackage[hyphens]{url}
\usepackage{graphicx}
\usepackage{natbib}
\usepackage{caption}
\usepackage{amsmath}
\usepackage{amssymb}
\usepackage{amsfonts}
\usepackage{amsthm}
\usepackage{mathtools}
\usepackage{dsfont}
\newcommand{\normalem}{}
\newcommand{\ULforem}{}
\usepackage{booktabs}
\usepackage{multirow}
\usepackage{array}
\usepackage{colortbl}
\usepackage{xcolor}
\usepackage{subcaption}
\usepackage{enumitem}
\usepackage{multicol}

\usepackage[linesnumbered,ruled,vlined]{algorithm2e}
\SetKwInput{KwInput}{Input}
\SetKwInput{KwOutput}{Output}
\SetKwRepeat{Do}{do}{while}

\SetCommentSty{mycommfont}

\DeclareMathOperator*{\argmax}{arg\,max}
\DeclareMathOperator*{\argmin}{arg\,min}

\theoremstyle{plain}
\newtheorem{theorem}{Theorem}

\theoremstyle{definition}

\definecolor{darkred}{rgb}{0.55, 0.0, 0.0}
\definecolor{royalblue}{rgb}{0.25, 0.41, 0.88}
\definecolor{orange}{rgb}{0.85, 0.45, 0.0}
\definecolor{gray}{rgb}{0.5, 0.5, 0.5}

\title{Towards Effective Federated Multimodal Graph Learning via Navigating Multifaceted Heterogeneity}

\author{
    Yinlin Zhu\textsuperscript{1},
    Di Wu\textsuperscript{1}\thanks{Corresponding author.},
    Yi Zhang\textsuperscript{2},
    Xunkai Li\textsuperscript{3},\\
    Wang Luo\textsuperscript{1},
    Wei-Jin Huang\textsuperscript{1},
    Miao Hu\textsuperscript{1},
    Guocong Quan\textsuperscript{1}
}
\affiliations{

    Sun Yat-sen University, Guangzhou, China\textsuperscript{\rm 1}\\
    Shandong University, Weihai, China\textsuperscript{\rm 2}\\
    Beijing Institute of Technology, Beijing, China\textsuperscript{\rm 3}\\

    \{zhuylin27,\ luow69, huangwj235\}@mail2.sysu.edu.cn, 
    hanziyu@mail.sdu.edu.cn, cs.xunkai.li@gmail.com,\\
    \{wudi27,\ humiao5,\ quangc\}@mail.sysu.edu.cn
    
}

\begin{document}
\maketitle

\begin{abstract}
    Multimodal-attributed graphs (MAGs), where nodes carry heterogeneous semantic content across multiple modalities while edges encode relational dependencies, have been widely adopted across diverse domains. Federated multimodal graph learning (FMGL) extends federated graph learning (FGL) to MAGs, enabling collaborative optimization across decentralized MAGs without exposing raw data. However, naively applying existing FGL methods to FMGL is insufficient, as they fail to navigate the \textit{multifaceted heterogeneity} inherent in decentralized MAGs, including task heterogeneity across diverse client objectives, modality heterogeneity from discrepant modality quality and semantic domains, and topology heterogeneity arising from divergent topological patterns with low cross-modality correlation. To address these challenges, we propose \underline{\textbf{Fed}}erated multimodal graph learning with \underline{\textbf{T}}opology-aware \underline{\textbf{C}}ross-modal \underline{\textbf{R}}outing (\textbf{FedTCR}), the first systematic algorithm designed for FMGL. To handle task heterogeneity, FedTCR employs a two-stage paradigm that comprises federated task-agnostic pre-training followed by isolated task-oriented fine-tuning. To jointly address modality and topology heterogeneity, FedTCR introduces a topology-aware cross-modal routing mechanism. Concretely, each client distills modality-specific knowledge into compact prototypes via topology-aware importance-weighted aggregation informed by graph structure; the server then evaluates cross-client cross-modal relationships among these structure-informed prototypes and routes informative ones as contrastive references, driving a tri-level cross-modal contrastive learning scheme that jointly aligns cross-client modalities while preserving discrimination. Experiments across 7 domains demonstrate that FedTCR outperforms state-of-the-art baselines on both graph-centric and modality-centric tasks.
\end{abstract}

\section{Introduction}
\label{sec: introduction}

        Multimodal-attributed graphs (MAGs) represent entities as nodes enriched with heterogeneous semantic content across multiple modalities, while edges encode relational dependencies among these entities. This expressive data structure has been widely adopted across diverse domains~\cite{mag_app_finance, mag_app_bio}. Motivated by such scenarios, multimodal graph learning (MGL) has garnered increasing attention, both enhancing graph-centric tasks through rich multimodal attributes~\cite{dgf} and improving modality-specific tasks via topological structures~\cite{mai2023multimodal}.

        Despite these advances, existing MGL approaches predominantly assume centralized data access, where all MAG data reside within a single learning system. In practice, however, multiple organizations (e.g., Instagram, X) independently maintain private MAGs that encode complementary knowledge, while privacy regulations and commercial constraints prohibit direct data aggregation, preventing isolated MGL models from leveraging this cross-platform collective intelligence.
    
        To this end, recent work~\cite{mm_openfgl} proposed federated multimodal graph learning (FMGL), extending federated graph learning (FGL)~\cite{openfgl} to MAGs. FGL enables collaborative optimization through iterative client-server communication, transmitting model parameters or gradients without exposing raw graph data. While conceptually promising, this paradigm lacks systematic analysis of key challenges and targeted solutions. We argue that naively extending traditional FGL methods to FMGL is insufficient, as they fail to address the \textit{multifaceted heterogeneity} inherent in decentralized MAGs:

        \paragraph{Task Heterogeneity.} The rich multimodal semantics of MAGs naturally lead to diverse application requirements across clients. Beyond conventional graph-level tasks (e.g., node classification), clients may pursue modality-centric objectives such as cross-modal retrieval or multimodal generation. However, mainstream FGL algorithms~\cite{fedsage_plus, fedpub, fgssl} are designed for a uniform task type and exclusively focus on graph-centric objectives~\cite{openfgl, fedgfm_plus}. Their architectures and optimization procedures are tightly coupled to a single task type, preventing collaboration when clients pursue different objectives.

        \paragraph{Modality Heterogeneity.} Distinct MAG sources and collection preferences across clients yield significant discrepancies in both modality quality and semantic domains. Clients with noisy or domain-biased modalities may propagate misleading knowledge during cross-client aggregation, degrading others' high-quality representations. Existing FGL methods address only unimodal attribute heterogeneity, lacking per-modality reliability assessment and cross-client cross-modal alignment capabilities.

        \paragraph{Topology Heterogeneity.} Graph topology provides crucial structural priors for modeling entity relationships. In FMGL, however, client MAGs often exhibit diverse topological patterns due to varying user behaviors or data collection processes. Existing FGL methods evaluate topological quality via homophily-based metrics on local subgraphs~\cite{fedtad, fedgta, fgssl}, implicitly assuming consistent homophily patterns across clients. This assumption is fragile in MAGs, where node pairs strongly homophilic under one modality may display weak homophily or even heterophily under another~\cite{dgf}, making traditional FGL methods unable to reliably distill and share topology priors across clients.

        Building upon these insights, we propose \underline{\textbf{Fed}}erated multimodal graph learning with \underline{\textbf{T}}opology-aware \underline{\textbf{C}}ross-modal \underline{\textbf{R}}outing (\textbf{FedTCR}), the first systematic algorithm addressing multifaceted heterogeneity in FMGL. To tackle \textit{\textbf{task heterogeneity}}, FedTCR adopts a two-stage paradigm: federated task-agnostic pre-training followed by isolated task-oriented fine-tuning, where clients collaboratively learn a shared multimodal graph encoder without assuming unified downstream objectives and then independently adapt it to their specific applications. To jointly address \textit{\textbf{modality and topology heterogeneity}}, FedTCR introduces a topology-aware cross-modal routing mechanism that synergistically integrates two key designs: (1) \textit{topology-aware prototype construction}, where each client estimates node importance from graph topology and performs importance-weighted aggregation to distill modality-specific knowledge into compact prototypes; (2) \textit{cross-modal prototype routing}, where the server evaluates cross-client cross-modal relationships among these structure-informed prototypes and routes informative ones as positive or negative references for each client. The routed references drive a tri-level cross-modal contrastive learning scheme (operating at node, neighbor, and client levels) that jointly aligns cross-client modalities while preserving discrimination.

        \paragraph{Contributions.} (1) \textbf{Valuable Insights.} We identify the \textit{multifaceted heterogeneity} in FMGL, revealing fundamental challenges overlooked by existing FGL methods. (2) \textbf{Novel Method.} We propose FedTCR, the first optimization algorithm explicitly designed for FMGL. FedTCR adopts a two-stage paradigm with task-agnostic federated pre-training to handle task heterogeneity, and introduces a topology-aware cross-modal routing mechanism to jointly address modality and topology heterogeneity. (3) \textbf{State-of-the-art Performance.} Experiments across 7 domains demonstrate that FedTCR achieves consistent and substantial improvements over state-of-the-art baselines on both graph-centric and modality-centric tasks.
\section{Related Work}
\label{sec: related work}

\paragraph{Multimodal Graph Learning.} MGL enriches graph learning with heterogeneous semantic content from multiple modalities. A line of work constructs modality-specific relational structures and fuses them for recommendation, e.g., MMGCN~\cite{mmgcn}, MGAT~\cite{mgat}, and LGMRec~\cite{lgmrec}. Another line exploits the interplay between multimodal semantics and topology for graph-oriented objectives~\cite{dmgc, dgf, mai2023multimodal}. More recent studies inject graph context into generative models, conditioning LLMs on multimodal neighborhoods for graph-to-text generation~\cite{yoon2023mmgl} and steering diffusion models with graph-guided signals for graph-to-image synthesis~\cite{instructg2i}. Complementary benchmarks such as OpenMAG~\cite{openmag} systematize MAG datasets, encoders, and task protocols. Despite this progress, existing MGL studies predominantly assume centralized data access and cannot accommodate privately held MAGs distributed across organizations.

\paragraph{Federated Graph Learning.} FGL extends federated learning to graph-structured data, enabling collaborative training without exposing raw graphs~\cite{fu2022fgl_survey_1}. Existing studies mainly fall into three categories: graph-level FGL, which mitigates cross-dataset heterogeneity via gradient-aware client clustering (GCFL+~\cite{gcfl_plus}); subgraph-level FGL, which addresses subgraph heterogeneity and missing cross-client links (Fed-PUB~\cite{fedpub}, FedGTA~\cite{fedgta}, FedSage+~\cite{fedsage_plus}, AdaFGL~\cite{li2024adafgl}, FGSSL~\cite{fgssl}, FGGP~\cite{fggp}, FedTAD~\cite{fedtad}, FedDEP~\cite{zhang2024fgl_feddep}, FedGCN~\cite{yao2024fgl_fedgcn}, FedGL~\cite{chen2021fedgl}); and node-level FGL over ego-networks (FedEgo~\cite{zhang2023node_level_fgl_fedego}), standardized by benchmarks such as FedGraphNN~\cite{he2021fedgraphnn}, FederatedScope-GNN~\cite{WangFedScope_22_fsg}, and OpenFGL~\cite{openfgl}. However, these methods are designed for unimodal graphs and a uniform task type, and thus cannot navigate the multifaceted heterogeneity inherent in federated multimodal graph learning, which is the focus of this work.

\section{Preliminaries}
\label{sec: preliminaries}

\paragraph{Multimodal-attributed Graphs.}
We consider a FMGL system consisting of $K$ clients coordinated by a central server. Each client $k$ privately maintains a MAG
$\mathcal{G}^k = (\mathcal{V}^k, \mathcal{E}^k, \{\mathbf{X}^{k,(m)}\}_{m\in\mathcal{M}})$, where $\mathcal{V}^k$ and $\mathcal{E}^k$ denote the node and edge sets, respectively.
$\mathcal{M}$ denotes the set of available modalities.
For each node $v_i \in \mathcal{V}^k$ under modality $m$, a modality-specific feature vector
$\mathbf{x}_i^{k,(m)} \in \mathbb{R}^{d_m}$ is obtained using a pre-trained modality encoder
(e.g., Sentence-BERT~\cite{sentence_bert} or T5~\cite{t5} for text, ViT~\cite{vit} or DINOv2~\cite{dinov2} for images).
The node features of modality $m$ are organized as a matrix
$\mathbf{X}^{k,(m)} \in \mathbb{R}^{|\mathcal{V}^k|\times d_m}$. All modalities share a common relational structure represented by the adjacency matrix
$\mathbf{A}^k \in \mathbb{R}^{|\mathcal{V}^k|\times|\mathcal{V}^k|}$,
with degree matrix $\mathbf{D}^k$ defined as
$\mathbf{D}^k_{ii}=\sum_j \mathbf{A}^k_{ij}$ and $\mathbf{D}^k_{ij}=0$ for $i\neq j$.
Following standard graph signal processing,
the symmetrically normalized adjacency matrix is defined as
$\tilde{\mathbf{A}}^k=(\mathbf{D}^k)^{-1/2}\mathbf{A}^k(\mathbf{D}^k)^{-1/2}$,
and the corresponding graph Laplacian is
$\mathbf{L}^k=\mathbf{I}-\tilde{\mathbf{A}}^k$.

\paragraph{Federated Training Protocol.} Due to privacy constraints, clients cannot share raw graph data; instead, they collaborate through iterative server-client communication. FedTCR adopts a two-stage paradigm: \textit{federated task-agnostic pre-training} followed by \textit{isolated task-oriented fine-tuning}. During pre-training, at each communication round the server samples a subset $\mathcal{S}$ of clients; each selected client $k \in \mathcal{S}$ downloads the global parameters $\hat{\mathbf{W}}$ and any server-side messages (e.g., routed prototypes) and performs local optimization on its private MAG $\mathcal{G}^k$; the server then aggregates the updated parameters via node-count-weighted averaging, $\hat{\mathbf{W}} = \sum_{k\in \mathcal{S}} \frac{|\mathcal{V}^k|}{N} \mathbf{W}^k$ with $N = \sum_{k\in \mathcal{S}}|\mathcal{V}^k|$, and performs additional procedures (e.g., cross-modal prototype routing). After pre-training converges, each client independently fine-tunes the encoder with a task-specific head without further communication.

\paragraph{Downstream Tasks.} We consider two graph-centric tasks (node classification and link prediction) and two modality-centric tasks (modality retrieval and modality generation covering G2Text and G2Image), with formal definitions deferred to Appendix~\ref{appendix: more experimental setups}. In this work, we focus on learning a task-agnostic multimodal graph encoder through federated pre-training, which can be subsequently adapted to these diverse downstream applications on individual clients.

\section{Methodology}
\label{sec: methodology}

We present FedTCR (Fig.~\ref{fig: framework}), an optimization algorithm explicitly designed for FMGL to address the multifaceted heterogeneity in Sec.~\ref{sec: introduction}. FedTCR adopts a two-stage paradigm to handle task heterogeneity: federated task-agnostic pre-training (Sec.~\ref{sec: Federated Task-agnostic Pre-training}) followed by isolated task-oriented fine-tuning (Sec.~\ref{sec: Isolated Task-oriented Fine-tuning}), and jointly addresses modality and topology heterogeneity during pre-training through a topology-aware cross-modal routing mechanism.

\begin{figure*}[htb]
 \centering
 \includegraphics[width=0.998\textwidth]{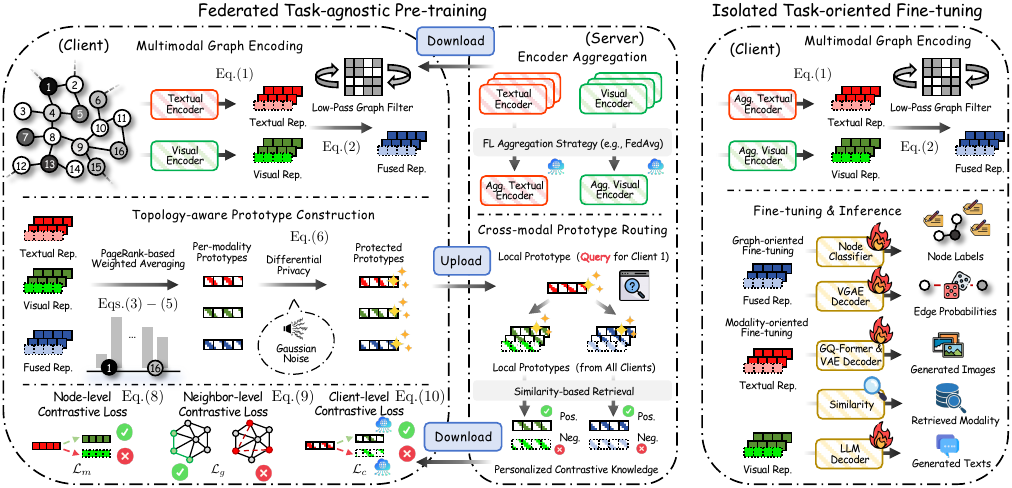}
  \caption{Overview of FedTCR, which addresses multifaceted heterogeneity in FMGL through federated task-agnostic pre-training with topology-aware cross-modal routing, followed by isolated task-oriented fine-tuning.}
\label{fig: framework}
\end{figure*}

\subsection{Federated Task-agnostic Pre-training}
\label{sec: Federated Task-agnostic Pre-training}

To address task heterogeneity, FedTCR decouples task-agnostic representation learning from task-specific adaptation. During pre-training, clients collaboratively learn a shared multimodal graph encoder without assuming unified downstream objectives, thereby capturing transferable multimodal and structural knowledge across decentralized MAGs. At each communication round, pre-training proceeds in two stages: (1) Each client performs local multimodal graph encoding, constructs topology-aware modality prototypes, and optimizes multimodal representations; (2) The server aggregates encoder parameters via weighted averaging and performs cross-modal prototype routing to deliver personalized contrastive knowledge to each client for the next round of optimization.

\paragraph{Multimodal Graph Encoding.}
For each client $k$ with MAG $\mathcal{G}^k$, we first project each modality's raw features into a shared $d$-dimensional latent space via modality-specific linear transformations $\mathbf{\Phi}^{(m)} \in \mathbb{R}^{d_m \times d}$:
\begin{equation}
\label{eq: modality projection}
\mathbf{z}_i^{k,(m)} = \frac{(\mathbf{X}^{k,(m)} \mathbf{\Phi}^{(m)})_i}{\|(\mathbf{X}^{k,(m)} \mathbf{\Phi}^{(m)})_i\|_2}, \quad \forall i \in \mathcal{V}^k,\, m \in \mathcal{M},
\end{equation}
where $(\cdot)_i$ denotes the $i$-th row, and $\mathbf{Z}^{k,(m)} \in \mathbb{R}^{|\mathcal{V}^k| \times d}$ stacks these rows to stabilize training. The fused representation is obtained by averaging across modalities: $\bar{\mathbf{Z}}^k = \frac{1}{|\mathcal{M}|}\sum_{m \in \mathcal{M}} \mathbf{Z}^{k,(m)}$.
We then apply low-pass graph filtering over the local topology to produce the multimodal graph representation:
\begin{equation}
\label{eq: graph filter}
\mathbf{H}^k = \frac{1}{\alpha+1}\sum_{l=0}^{L}\left(\frac{\alpha}{\alpha+1}\tilde{\mathbf{A}}^k\right)^l \bar{\mathbf{Z}}^k,
\end{equation}
where $\alpha > 0$ controls the smoothing strength, $\tilde{\mathbf{A}}^k$ is the symmetrically normalized adjacency matrix, and $L$ is the propagation depth. As we show in Theorem~\ref{thm: smooth fused rep}, many widely used GNNs (e.g., GCN~\cite{gcn}, SGC~\cite{sgc}) can be viewed as special cases of this graph-filtering formulation.

\paragraph{Topology-aware Prototype Construction.}
In FMGL, we distill each client's modality-specific knowledge into compact prototypes for cross-client alignment. Naive mean aggregation treats all nodes equally and ignores the local MAG topology; inspired by~\cite{louvain,random_walk}, we instead assign each node a topology-aware weight measuring its representativeness for the graph's community. 
Without loss of generality, we use PageRank algorithm~\cite{pagerank} to compute these weights, so that importance propagates over multiple hops rather than depending only on 1-hop degree. Let $\alpha \in (0,1)$ be the damping factor; the PageRank vector $\boldsymbol{\pi}^k \in \mathbb{R}^{|\mathcal{V}^k|}$ is the stationary distribution
\begin{equation}
\label{eq: pagerank stationary}
\boldsymbol{\pi}^k = \frac{1-\alpha}{|\mathcal{V}^k|}\,\mathbf{1} + \alpha \left((\mathbf{D}^k)^{-1} \mathbf{A}^k\right)^\top \boldsymbol{\pi}^k,
\end{equation}
where $\mathbf{A}^k$ is the adjacency matrix and $\mathbf{D}^k$ is the diagonal degree matrix; the transpose gives the transition probabilities from neighbors into each node. We compute $\boldsymbol{\pi}^k$ by power iteration, then set $\omega_i^k = \pi_i^k$ after convergence, formulated as:
\begin{equation}
\label{eq: node importance}
\pi_i^{k,(t+1)} = \frac{1-\alpha}{|\mathcal{V}^k|} + \alpha \sum_{j \in \mathcal{V}^k} \frac{\mathbf{A}^k_{ji}}{d_j^k} \pi_j^{k,(t)},
\end{equation}
where $d_j^k = \sum_{l} \mathbf{A}^k_{jl}$ is the degree of node $v_j$. $\boldsymbol{\pi}^k$ is normalized so that $\sum_{i \in \mathcal{V}^k} \omega_i^k = 1$ with $\omega_i^k = \pi_i^k$.

Using these normalized weights, we define modality-specific prototypes by importance-weighted aggregation:
\begin{equation}
\label{eq: prototype construction}
\mathbf{P}^{k,(m)} = \sum_{i \in \mathcal{V}^k} \omega_i^k \cdot \mathbf{z}_i^{k,(m)}, \quad \forall m \in \bar{\mathcal{M}},
\end{equation}
where $\bar{\mathcal{M}} = \mathcal{M} \cup \{|\mathcal{M}|+1\}$ adds the fused modality; $\mathbf{P}^{k,(m)}$ for $m \in \bar{\mathcal{M}}$ is computed from $\mathbf{Z}^{k,(m)}$, and $\mathbf{Z}^{k,(|\mathcal{M}|+1)}=\mathbf{H}^k$ (i.e., Eq.~\eqref{eq: graph filter}). 
Each prototype thus distills the local MAG's modality-specific semantics and structure into a single vector. 
Together with the updated encoder parameters, they are uploaded to the server for aggregation and used in the subsequent routing. 
To protect privacy, each client further applies $(\varepsilon,\delta)$-differential privacy (DP)~\cite{dwork2008differential,abadi2016deep} before uploading: each prototype $\mathbf{P}^{k,(m)}$ is L2-clipped with a bound $C > 0$ and then perturbed with Gaussian noise calibrated to the target $(\varepsilon,\delta)$:
\begin{equation}
\label{eq: dp clip}
\begin{aligned}
\tilde{\mathbf{P}}^{k,(m)} &= \mathbf{P}^{k,(m)} \cdot \min\Bigl(1, \frac{C}{\|\mathbf{P}^{k,(m)}\|_2}\Bigr), \\
\hat{\mathbf{P}}^{k,(m)} &= \tilde{\mathbf{P}}^{k,(m)} + \mathcal{N}(0, \sigma^2 C^2 \mathbf{I}_d),
\end{aligned}
\end{equation}
where $\sigma$ is the noise scale. The server performs cross-modal prototype routing (Eq.~\eqref{eq: prototype routing}) on the noisy prototypes, and $\varepsilon$ and $C$ jointly control the privacy-utility tradeoff.

\paragraph{Cross-modal Prototype Routing.}
After receiving all clients' topology-aware prototypes and encoder parameters, the server (1) aggregates encoder parameters $\{\mathbf{\Phi}^{(m)}\}_{m \in \mathcal{M}}$ via a mainstream FL algorithm (e.g., FedAvg~\cite{fedavg}), and (2) performs cross-modal prototype routing to produce and deliver personalized contrastive knowledge to each client, which guides the next round of federated optimization.

For each client $a$, modality $i \in \bar{\mathcal{M}}$, and each other modality $j \neq i$, the server selects from $\{\mathbf{P}^{b,(j)}\}_{b \neq a}$ the most similar prototype as the positive and the most dissimilar as the negative, which can be formulated as follows:
\begin{equation}
\label{eq: prototype routing}
\begin{aligned}
\mathbf{P}_{\text{pos}}^{a,(i,j)} &= \hat{\mathbf{P}}^{b^*,(j)}, \quad b^* = \argmax_{b \neq a} \cos(\hat{\mathbf{P}}^{a,(i)}, \hat{\mathbf{P}}^{b,(j)}), \\
\mathbf{P}_{\text{neg}}^{a,(i,j)} &= \hat{\mathbf{P}}^{b',(j)}, \quad b' = \argmin_{b \neq a} \cos(\hat{\mathbf{P}}^{a,(i)}, \hat{\mathbf{P}}^{b,(j)}).
\end{aligned}
\end{equation}
Here $\cos(\cdot,\cdot)$ denotes cosine similarity. The server sends the routed pairs $\{(\mathbf{P}_{\text{pos}}^{a,(i,j)}, \mathbf{P}_{\text{neg}}^{a,(i,j)})\}_{i,j}$ to each client $a$ for client-level contrastive learning (Eq.~\eqref{eq: cross client loss}). Enforcing this contrast yields two effects: (1) per-modality cross-client alignment encourages complementary knowledge transfer and mitigates propagation of noisy or domain-biased modalities, addressing modality heterogeneity; (2) aligning topology-aware prototypes across modality boundaries harmonizes diverse topological patterns, addressing the topology heterogeneity inherent in decentralized MAGs.

\paragraph{Local Optimization Objectives.}
Each client optimizes a composite objective comprising two local multimodal graph learning losses and one cross-client alignment loss, following the cross-contrastive paradigm. Letting $\psi(\mathbf{u},\mathbf{v})=\exp(\cos(\mathbf{u},\mathbf{v})/\tau)$ with temperature $\tau$, and denoting $\mathbf{z}_i^{k,(|\mathcal{M}|+1)} = \mathbf{h}_i^k$ as the fused graph-filtered representation treated as an additional modality, we write $\psi_{ij}^{(m_1,m_2)}$ as shorthand for $\psi\big(\mathbf{z}_i^{k,(m_1)}, \mathbf{z}_j^{k,(m_2)}\big)$:

(1) \textit{\textbf{Node-level Contrastive Loss}} maximizes the inter-modality agreement of the same node while contrasting it against imposter nodes from the other modality:
\begin{equation}
\label{eq: modal alignment}
\mathcal{L}_m = \sum_{\substack{m_1 \in \bar{\mathcal{M}} \\ m_2 \in \bar{\mathcal{M}} \setminus \{m_1\}}} \sum_{i \in \mathcal{V}^k} -\log \frac{\psi_{ii}^{(m_1,m_2)}}{\sum_{j \in \mathcal{V}^k} \psi_{ij}^{(m_1,m_2)}}.
\end{equation}

(2) \textit{\textbf{Neighbor-level Contrastive Loss}} captures the semantic relations underlying the graph topology. For each node $v_i \in \mathcal{V}^k$, we simulate random walks on client $k$'s graph to collect a set of topological neighbors $\mathcal{P}_i^k$ as positives, and randomly sample $\mathcal{N}_i^k \subset \mathcal{V}^k \setminus \mathcal{P}_i^k$ as negatives:
\begin{equation}
\label{eq: graph contrastive}
\mathcal{L}_g = \sum_{\substack{m_1 \in \bar{\mathcal{M}} \\ m_2 \in \bar{\mathcal{M}} \setminus \{m_1\}}} \sum_{i \in \mathcal{V}^k} -\log \frac{\sum_{j \in \mathcal{P}_i^k} \psi_{ij}^{(m_1,m_2)}}{\sum_{j \in \mathcal{P}_i^k \cup \mathcal{N}_i^k} \psi_{ij}^{(m_1,m_2)}}.
\end{equation}

(3) \textbf{Client-level Contrastive Loss} leverages the personalized contrastive knowledge from the server (Eq.~\eqref{eq: prototype routing}) to perform prototype-level contrastive learning across clients and modalities. Writing $\psi_{\mathrm{pos}}^{(i,j)} = \psi\big(\mathbf{P}^{k,(i)}, \mathbf{P}_{\text{pos}}^{k,(i,j)}\big)$ and $\psi_{\mathrm{neg}}^{(i,j)} = \psi\big(\mathbf{P}^{k,(i)}, \mathbf{P}_{\text{neg}}^{k,(i,j)}\big)$:
\begin{equation}
\label{eq: cross client loss}
\mathcal{L}_c = \sum_{\substack{i \in \bar{\mathcal{M}} \\ j \in \bar{\mathcal{M}} \setminus \{i\}}} -\log \frac{\psi_{\mathrm{pos}}^{(i,j)}}{\psi_{\mathrm{pos}}^{(i,j)} + \psi_{\mathrm{neg}}^{(i,j)}}.
\end{equation}
Notably, Eq.~\eqref{eq: cross client loss} is activated from the second round after the server collects all client prototypes.

Finally, the overall client-side optimization objective during pre-training is formulated as follows:
\begin{equation}
\label{eq: total loss}
\mathcal{L} = \lambda_m \mathcal{L}_m + \lambda_g \mathcal{L}_g + \lambda_c \mathcal{L}_c,
\end{equation}
where $\lambda_m$, $\lambda_g$, and $\lambda_c$ are balancing coefficients.

\subsection{Isolated Task-oriented Fine-tuning}
\label{sec: Isolated Task-oriented Fine-tuning}

Once federated pre-training converges, each client retains the aggregated encoder and independently fine-tunes on its local MAG without further communication. Client $k$ attaches a task-specific head $g^k$ and optimizes it jointly with the encoder via a task-specific loss $\mathcal{L}_{\text{task}}$: for \textit{graph-centric tasks} (e.g., node classification, link prediction), the head takes the graph-filtered representation $\mathbf{H}^k$ (Eq.~\eqref{eq: graph filter}); for \textit{modality-centric tasks} (e.g., modality retrieval), it takes the modality-specific embeddings $\mathbf{Z}^{k,(m)}$ (Eq.~\eqref{eq: modality projection}) to retain modality-level characteristics. This isolation is deliberate: it avoids conflicting gradients from heterogeneous downstream objectives while preserving the transferable knowledge from pre-training. The complete procedure is given in Appendix~\ref{appendix: pseudocode}.
\section{Experiments}
\label{sec: experiments}

We evaluate FedTCR along six research questions, with full reproducibility details deferred to Appendix~\ref{appendix: dataset details} and Appendix~\ref{appendix: more experimental setups}:
\textbf{Q1}: Can FedTCR outperform existing methods on graph-centric and modality-centric tasks under homogeneous-task clients (Sec.~\ref{sec: main result})?
\textbf{Q2}: Can it enable effective collaboration across task-heterogeneous clients (Sec.~\ref{sec: task hetero})?
\textbf{Q3}: What does each module of FedTCR contribute (Sec.~\ref{sec: ablation study})?
\textbf{Q4}: How robust is it under data sparsity (Sec.~\ref{sec: robustness})?
\textbf{Q5}: How sensitive is it to key hyperparameters (Sec.~\ref{sec: sensitivity analysis})?
\textbf{Q6}: How does it perform under differentially private prototype perturbation (Sec.~\ref{sec: privacy utility})?

\subsection{Experimental Setup}
\label{sec: experimental setup}

\paragraph{Datasets and Simulation Strategy.} We evaluate FedTCR on 8 publicly available MAG datasets spanning 7 domains: Movies~\cite{ni2019justifying}, Grocery and Toys~\cite{ni2019justifying,hou2024bridging}, RedditS~\cite{RedditS}, BiliDance~\cite{zhang2024ninerec}, ele-fashion~\cite{ni2019justifying,hou2024bridging}, Flickr30k~\cite{plummer2015_Flickr30k}, and SemArt~\cite{garcia2018_SemArt}. Following subgraph FL benchmarks~\cite{openfgl}, each MAG is partitioned into client subgraphs with the Louvain algorithm~\cite{louvain}: 5 clients per MAG in the \textit{homogeneous-task} configuration, and 4 clients per MAG in the \textit{heterogeneous-task} one (one MAG per task, 20 clients across the 5 tasks in total). Dataset statistics and descriptions are in Appendix~\ref{appendix: dataset details}.

\paragraph{Baselines.} We compare against four method families: (1) \textit{Isolated Training} (MMGCN~\cite{mmgcn}, MGAT~\cite{mgat}, LGMRec~\cite{lgmrec}); (2) \textit{Classical FL} (FedAvg~\cite{fedavg}, SCAFFOLD~\cite{scaffold}, FedDC~\cite{feddc}) with an MMGCN backbone; (3) \textit{Multimodal FL} (PEPSY~\cite{pepsy}, FedMVP~\cite{fedmvp}, FedMAC~\cite{fedmac}); and (4) \textit{Unimodal FGL} (FedSage+~\cite{fedsage_plus}, FedGTA~\cite{fedgta}, FedSPA~\cite{fedspa}, FedIIH~\cite{fediih}, FedSSP~\cite{fedssp}, S2FGL~\cite{s2fgl}, FedLap~\cite{fedlap}). Unimodal FGL methods are graph-centric only and thus inapplicable to modality-centric tasks; label-dependent methods (e.g., FedGTA, FedIIH) cannot perform link prediction. These cases are marked as ``N/A''. Baseline details are in Appendix~\ref{appendix: baseline details}.

\paragraph{Evaluation Metrics.} We adopt task-specific metrics: ACC for node classification, AUC for link prediction, R@1 for modality retrieval, and R-L (ROUGE-L) for G2Text with C-S (CLIP-Score) for G2Image.

\begin{table*}[t]
    \centering
    \resizebox{\textwidth}{!}{
    \begin{tabular}{lcc|cc|cc|cc}
    \toprule
      & \multicolumn{2}{c|}{Node Classification (ACC)} & \multicolumn{2}{c|}{Link Prediction (AUC)} & \multicolumn{2}{c|}{Modality Retrieval (R@1)} & \multicolumn{2}{c}{Modality Generation (R-L/C-S)} \\
    \cmidrule(lr){2-3} \cmidrule(lr){4-5} \cmidrule(lr){6-7} \cmidrule(lr){8-9}
    Methods & Movies & Grocery & RedditS & BiliDance & Toys & ele-fashion & Flickr30k & SemArt \\
    \midrule
    MMGCN & \underline{53.47$\pm$0.32} & 80.39$\pm$0.51 & 89.52$\pm$0.23 & 76.06$\pm$0.42 & 49.29$\pm$0.46 & 81.74$\pm$0.28 & 24.56$\pm$0.23 & 66.12$\pm$0.28 \\ 
    MGAT & 52.12$\pm$0.34 & 80.41$\pm$0.49 & 90.23$\pm$0.25 & 76.21$\pm$0.44 & 48.89$\pm$0.43 & 82.12$\pm$0.31 & 24.34$\pm$0.21 & 65.78$\pm$0.30 \\ 
    LGMRec & 53.34$\pm$0.29 & 80.56$\pm$0.53 & 90.89$\pm$0.22 & \underline{76.68$\pm$0.38} & \underline{49.78$\pm$0.45} & \underline{82.45$\pm$0.26} & 25.12$\pm$0.24 & 66.89$\pm$0.26 \\ 
    Isolated & 53.07$\pm$0.36 & \underline{80.93$\pm$0.57} & 88.34$\pm$0.26 & 62.89$\pm$0.46 & 49.44$\pm$0.52 & 81.37$\pm$0.31 & 23.95$\pm$0.26 & 65.34$\pm$0.31 \\ 
    \midrule
    FedAvg & 49.07$\pm$0.43 & 78.18$\pm$0.68 & 92.32$\pm$0.31 & 72.30$\pm$0.56 & 45.31$\pm$0.62 & 68.96$\pm$0.37 & 24.02$\pm$0.33 & 65.56$\pm$0.37 \\ 
    SCAFFOLD & 49.56$\pm$0.45 & 78.66$\pm$0.63 & 92.78$\pm$0.28 & 72.78$\pm$0.52 & 45.89$\pm$0.58 & 69.67$\pm$0.39 & 23.89$\pm$0.29 & 65.89$\pm$0.34 \\ 
    FedDC & 49.87$\pm$0.40 & 78.43$\pm$0.66 & 93.01$\pm$0.30 & 73.01$\pm$0.54 & 46.23$\pm$0.64 & 69.89$\pm$0.35 & 24.45$\pm$0.31 & 66.12$\pm$0.38 \\ 
    \midrule
    PEPSY & 48.67$\pm$0.47 & 78.13$\pm$0.73 & 93.45$\pm$0.33 & 73.56$\pm$0.60 & 47.21$\pm$0.66 & 70.12$\pm$0.40 & 24.78$\pm$0.35 & 66.45$\pm$0.40 \\ 
    FedMVP & 48.36$\pm$0.50 & 77.98$\pm$0.71 & 93.21$\pm$0.36 & 74.18$\pm$0.63 & 44.73$\pm$0.69 & 69.52$\pm$0.44 & 25.01$\pm$0.32 & 66.78$\pm$0.43 \\ 
    FedMAC & 49.01$\pm$0.44 & 78.35$\pm$0.76 & \underline{93.67$\pm$0.31} & 74.02$\pm$0.57 & 46.89$\pm$0.63 & 70.44$\pm$0.41 & \underline{25.23$\pm$0.34} & \underline{67.01$\pm$0.38} \\ 
    \midrule
    FedSage+ & 50.23$\pm$0.36 & 79.21$\pm$0.56 & 92.89$\pm$0.26 & 72.67$\pm$0.46 & N/A & N/A & N/A & N/A \\ 
    FedGTA & 50.89$\pm$0.32 & 79.56$\pm$0.51 & N/A & N/A & N/A & N/A & N/A & N/A \\ 
    FedSPA & 48.89$\pm$0.64 & 79.34$\pm$1.01 & 93.12$\pm$0.46 & 73.12$\pm$0.82 & N/A & N/A & N/A & N/A \\ 
    FedIIH & 51.03$\pm$0.34 & 79.62$\pm$0.53 & N/A & N/A & N/A & N/A & N/A & N/A \\
    FedSSP & 48.35$\pm$0.43 & 78.26$\pm$0.67 & 93.52$\pm$0.31 & 73.45$\pm$0.55 & 43.11$\pm$0.61 & 70.31$\pm$0.36 & N/A & N/A \\ 
    S2FGL & 48.56$\pm$0.89 & 79.12$\pm$1.38 & 92.97$\pm$0.63 & 72.89$\pm$1.13 & N/A & N/A & N/A & N/A \\ 
    FedLap & 50.77$\pm$0.30 & 79.43$\pm$0.47 & 93.23$\pm$0.22 & 73.23$\pm$0.39 & N/A & N/A & N/A & N/A \\ 
    \midrule
    \textbf{FedTCR (Ours)} & \textbf{54.27$\pm$0.25} & \textbf{81.01$\pm$0.40} & \textbf{97.84$\pm$0.18} & \textbf{78.45$\pm$0.32} & \textbf{53.64$\pm$0.36} & \textbf{83.87$\pm$0.22} & \textbf{26.89$\pm$0.18} & \textbf{69.45$\pm$0.22} \\ 
    \bottomrule
    \end{tabular}
    }
    \caption{Performance comparison under the homogeneous-task setting (5 clients). `N/A' indicates the algorithm is not applicable to the task. Best results are in bold and the second-best are underlined.}
    \label{tab: main result}
    \vspace{-0.1cm}
\end{table*}

\subsection{Homogeneous Tasks (Answer for \textbf{Q1})}
\label{sec: main result}

To answer \textbf{Q1}, we compare FedTCR against all four baseline families on the 8 MAGs under the homogeneous-task setting (5 clients per MAG, all sharing the MAG's designated task); Table~\ref{tab: main result} reports the full per-dataset comparison.

\paragraph{Graph-centric Tasks.}

FedTCR consistently achieves the best performance across all graph-centric benchmarks (Table~\ref{tab: main result}), improving over the second-best method by +1.50\% ACC on Movies and +0.10\% on Grocery for node classification, and surpassing the runner-up by +4.45\% AUC on RedditS and +2.31\% on BiliDance for link prediction. Classical FL methods with the MMGCN backbone outperform isolated training on the link prediction benchmarks, confirming the benefit of federated collaboration on structure-dependent tasks, but they cannot exploit modality- and topology-specific knowledge across clients as FedTCR does.

\paragraph{Modality-centric Tasks.}

FedTCR also demonstrates clear superiority on modality-centric tasks: +7.75\% R@1 on Toys and +1.72\% on ele-fashion for modality retrieval, and +6.58\% R-L on Flickr30k (G2Text) and +3.64\% C-S on SemArt (G2Image) for modality generation, over the best competing method. This consistent superiority shows that topology-aware cross-modal routing effectively transfers complementary multimodal knowledge across clients, while task-agnostic pre-training learns representations that generalize across downstream applications.

\begin{figure*}[htb]
\centering
 \includegraphics[width=0.998\textwidth]{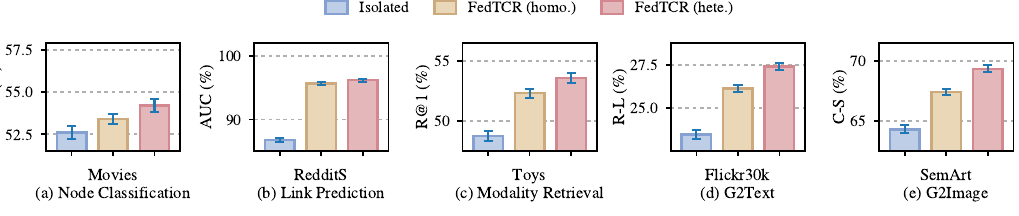}
 \caption{Performance under the task-heterogeneous setting (4 clients for each dataset, i.e., one MAG per task). FedTCR (homo.) federates only the 4 same-task clients, while FedTCR (hete.) pre-trains all 20 clients across the 5 tasks.}
\label{fig: task hetero}
\end{figure*}

\subsection{Heterogeneous Tasks (Answer for \textbf{Q2})}
\label{sec: task hetero}

To answer \textbf{Q2}, we evaluate FedTCR under a task-heterogeneous setting comprising 5 downstream tasks, each instantiated on a different MAG partitioned into 4 clients: node classification on Movies, link prediction on RedditS, modality retrieval on Toys, G2Text on Flickr30k, and G2Image on SemArt. Existing FL and FGL methods cannot join such a federation at all, as their architectures couple the federated parameters and the training objective to a single task type; their only viable collaboration is same-task federation, which is already covered by Table~\ref{tab: main result} where all of them trail FedTCR. We therefore conduct a controlled self-comparison that varies only the collaboration scope over identical per-task partitions: \textit{Isolated} (no federation), FedTCR (homo.) (4 same-task clients), and FedTCR (hete.) (all 20 clients). As shown in Fig.~\ref{fig: task hetero}, FedTCR (homo.) already improves over Isolated training by +7.06\% on average, and FedTCR (hete.) further improves over FedTCR (homo.) by +2.44\% on average. The gains are the largest on the modality-generation tasks (+3.87\% on average), which depend most heavily on the complementary cross-modal knowledge transferred by the topology-aware routing, and remain positive on graph-centric tasks (+1.00\%). Since the arms differ only in the collaboration scope, these results establish that cross-task collaboration carries genuine net benefit rather than interference, enabling cross-task collective intelligence that is architecturally out of reach for existing methods.

\subsection{Ablation Study (Answer for \textbf{Q3})}
\label{sec: ablation study}

To answer \textbf{Q3}, we ablate five components of FedTCR: (1) \textit{w/o TW}: replaces PageRank importance weighting with uniform averaging in prototype construction (Eq.~\eqref{eq: prototype construction}); (2) \textit{w/o CR}: removes cross-modal prototype routing and the associated client-level loss $\mathcal{L}_c$ (Eq.~\eqref{eq: cross client loss}); (3) \textit{w/o $\mathcal{L}_m$} and (4) \textit{w/o $\mathcal{L}_g$}: remove the node-level (Eq.~\eqref{eq: modal alignment}) and neighbor-level (Eq.~\eqref{eq: graph contrastive}) contrastive losses, respectively; and (5) \textit{w/o TS}: removes the pre-training stage and directly performs end-to-end federated task-specific training. All variants are trained under the identical protocol and hyperparameters as the full model, so any performance gap is attributable solely to the removed component.

\begin{table}
    \centering
    \resizebox{\columnwidth}{!}{
    \begin{tabular}{lccccc}
    \toprule
      & Node Cls. & Link Pred. & Retrieval & G2Text & G2Image \\
      & (ACC) & (AUC) & (R@1) & (R-L) & (C-S) \\
    \cmidrule(lr){2-2} \cmidrule(lr){3-3} \cmidrule(lr){4-4} \cmidrule(lr){5-5} \cmidrule(lr){6-6}
    Variants & Movies & RedditS & Toys & Flickr30k & SemArt \\
    \midrule
    w/o TW & 53.69$\pm$0.30 & 96.89$\pm$0.21 & 52.45$\pm$0.42 & 26.23$\pm$0.22 & 68.56$\pm$0.25 \\ 
    w/o CR & 53.03$\pm$0.33 & 95.67$\pm$0.24 & 51.12$\pm$0.47 & 25.56$\pm$0.26 & 68.04$\pm$0.28 \\ 
    w/o $\mathcal{L}_m$ & 53.92$\pm$0.28 & 97.12$\pm$0.20 & 52.91$\pm$0.40 & 26.45$\pm$0.19 & 68.78$\pm$0.24 \\ 
    w/o $\mathcal{L}_g$ & 53.58$\pm$0.31 & 96.45$\pm$0.22 & 52.01$\pm$0.44 & 26.01$\pm$0.23 & 68.34$\pm$0.26 \\ 
    w/o TS & 53.36$\pm$0.34 & 96.12$\pm$0.25 & 51.67$\pm$0.48 & 25.78$\pm$0.25 & 67.98$\pm$0.29 \\ 
    \midrule
    FedTCR (Full) & \textbf{54.27$\pm$0.25} & \textbf{97.84$\pm$0.18} & \textbf{53.64$\pm$0.36} & \textbf{26.89$\pm$0.18} & \textbf{69.45$\pm$0.22} \\ 
    \bottomrule
    \end{tabular}
    }
    \captionof{table}{Ablation study (5 clients). Each row removes one component from FedTCR.}
    \label{tab: ablation}
\end{table}

\begin{figure}[htb]
\centering

 \includegraphics[width=0.48\textwidth]{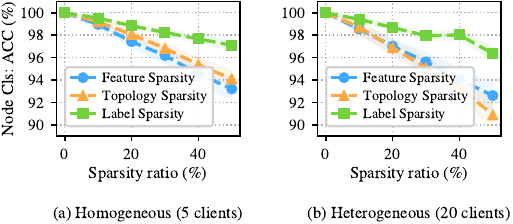}
 \caption{Retained ACC on Movies under three sparsity types at varying ratios, relative to the unperturbed setting. Shaded bands denote the standard deviation over five seeds.}
\label{fig: robustness}
\end{figure}

As shown in Table~\ref{tab: ablation}, removing any component consistently degrades performance, demonstrating their complementary contributions. \textit{w/o CR} incurs the largest drop (e.g., $-$2.28\% ACC on Movies and $-$4.70\% R@1 on Toys), highlighting the importance of cross-modal prototype routing for exchanging task-relevant knowledge beyond parameter aggregation. The degradation of \textit{w/o TS} further verifies that task-agnostic pre-training is crucial for learning transferable representations before task-specific adaptation. In addition, \textit{w/o TW} consistently underperforms the full model, validating the benefit of topology-aware prototype construction, while the larger impact of \textit{w/o $\mathcal{L}_g$} compared with \textit{w/o $\mathcal{L}_m$} suggests that neighborhood-level supervision provides stronger guidance than node-level alignment.

\subsection{Robustness Analysis (Answer for \textbf{Q4})}
\label{sec: robustness}

To answer \textbf{Q4}, we evaluate FedTCR under three data sparsity scenarios controlled by ratio $\alpha$: \textit{feature sparsity} (feature masking), \textit{topology sparsity} (edge removal), and \textit{label sparsity} (label removal), in both task settings. As shown in Fig.~\ref{fig: robustness}, FedTCR remains robust across all sparsity types and task settings: at $\alpha=50\%$ in the homogeneous setting it preserves 93.2\%, 94.1\%, and 97.1\% of its original ACC under feature, topology, and label sparsity, respectively. The degradation pattern is setting-dependent: feature sparsity hurts most under homogeneous clients, where corrupted modality representations propagate directly into prototype construction, whereas topology sparsity dominates in the heterogeneous-task federation, where divergent objectives make topology-aware routing the main channel of cross-client support; label sparsity stays the mildest in both settings, since task-agnostic pre-training requires no downstream supervision.

\subsection{Hyperparameter Sensitivity (Answer for \textbf{Q5})}
\label{sec: sensitivity analysis}

\begin{table}
    \centering
    \resizebox{\columnwidth}{!}{
    \begin{tabular}{llccccc}
    \toprule
    \multicolumn{2}{l}{Hyperparameter} & \multicolumn{5}{c}{ Node Classification Accuracy (\%)} \\
    \midrule
    $\lambda_m$ & $\{0.1, 0.2, \mathbf{0.3}, 0.4, 0.5\}$ & 53.96 & 54.12 & \textbf{54.27} & 54.18 & 54.02 \\ 
    $\lambda_g$ & $\{0.05, 0.1, \mathbf{0.2}, 0.3, 0.5\}$ & 53.88 & 54.09 & \textbf{54.27} & 54.16 & 53.97 \\ 
    $\lambda_c$ & $\{0.1, 0.3, \mathbf{0.5}, 0.7, 1.0\}$ & 52.41 & 53.72 & \textbf{54.27} & 54.05 & 53.83 \\ 
    $\tau$ & $\{0.1, 0.3, \mathbf{0.5}, 0.7, 1.0\}$ & 54.03 & 54.19 & \textbf{54.27} & 53.86 & 52.58 \\ 
    $L$ & $\{1, 2, \mathbf{3}, 4, 5\}$ & 54.01 & 54.17 & \textbf{54.27} & 54.21 & 54.08 \\ 
    \bottomrule
    \end{tabular}
    }
    \captionof{table}{Hyperparameter sensitivity on Movies. Each row sweeps one hyperparameter over the grid shown, with the others fixed at their default values.}
    \label{tab: sensitivity}
\end{table}

To answer \textbf{Q5}, we sweep the loss weights $\lambda_m$, $\lambda_g$, $\lambda_c$ (summing to $1$ by default), the temperature $\tau$, and the propagation depth $L$. As shown in Table~\ref{tab: sensitivity}, FedTCR performs best at the default configuration and remains stable around it; the model is primarily sensitive to $\lambda_c$ and $\tau$: an overly small $\lambda_c$ weakens the cross-client alignment signal, while an overly large $\tau$ makes the contrastive distribution overly uniform. In contrast, varying $\lambda_m$, $\lambda_g$, or $L$ changes the accuracy only marginally, indicating that the defaults already lie in a flat optimal region and that FedTCR requires little hyperparameter tuning in practice.

\subsection{Privacy-Utility Tradeoff (Answer for \textbf{Q6})}
\label{sec: privacy utility}

\begin{table}
    \centering
    \resizebox{0.95\columnwidth}{!}{
    \begin{tabular}{cccc}
    \toprule
    Privacy budget $\varepsilon$ & Movies (ACC) & RedditS (AUC) & Toys (R@1) \\
    \midrule
    $\infty$ (w/o DP) & \textbf{54.27$\pm$0.25} & \textbf{97.84$\pm$0.18} & \textbf{53.64$\pm$0.36} \\ 
    16 & 54.14$\pm$0.28 & 97.76$\pm$0.20 & 53.56$\pm$0.39 \\ 
    8 & 53.98$\pm$0.30 & 97.61$\pm$0.22 & 53.42$\pm$0.42 \\ 
    4 & 53.73$\pm$0.34 & 97.34$\pm$0.26 & 53.14$\pm$0.46 \\ 
    2 & 53.31$\pm$0.37 & 96.93$\pm$0.29 & 52.78$\pm$0.50 \\ 
    1 & 52.89$\pm$0.43 & 96.34$\pm$0.35 & 52.25$\pm$0.55 \\
    \bottomrule
    \end{tabular}
    }
    \captionof{table}{Performance under $(\varepsilon, \delta)$-DP prototype perturbation with $\delta=10^{-5}$ and clipping bound $C=1.0$, mean and standard deviation over five random seeds (percentages).}
    \label{tab: privacy}
\end{table}

To answer \textbf{Q6}, we evaluate FedTCR with $(\varepsilon, \delta)$-differentially private prototype uploading (Eq.~\eqref{eq: dp clip}), fixing $\delta=10^{-5}$ and $C=1.0$ while varying $\varepsilon$. As shown in Table~\ref{tab: privacy}, performance is nearly lossless at $\varepsilon \ge 8$ (over 99.4\% retained), since prototypes are importance-weighted aggregates over many nodes, which limits per-client sensitivity, and the server-side aggregation further averages the perturbations out across clients. Even at $\varepsilon=1$, FedTCR retains over 97.4\% of its non-private performance, staying clearly ahead of the strongest non-private baseline on RedditS and Toys and within 1.1\% of MMGCN on Movies.

\section{Conclusion}
\label{sec: Conclusion}

We identify the \textit{multifaceted heterogeneity} inherent in federated multimodal graph learning (FMGL), spanning task, modality, and topology heterogeneity across decentralized MAGs. We propose FedTCR, the first systematic algorithm for FMGL, which employs a two-stage paradigm (federated task-agnostic pre-training followed by isolated task-oriented fine-tuning) and a topology-aware cross-modal routing mechanism that distills modality-specific knowledge into structure-informed prototypes for cross-client contrastive alignment. Experiments across 8 datasets show that FedTCR consistently achieves state-of-the-art performance on graph-centric and modality-centric tasks, laying a foundation for privacy-preserving multimodal graph learning.

\bibliography{citation}
\cleardoublepage

\appendix
\newpage

\section{More Related Works}

\paragraph{Graph Neural Networks (GNNs).} Earlier research on deep graph learning extends convolution to handle graphs~\cite{bruna2013spectral} but comes with notable parameter counts. To this end, GCN~\cite{gcn} simplifies graph convolution by utilizing a 1-order Chebyshev filter to capture local neighborhood information. Moreover, GAT~\cite{gat} adopts graph attention, allowing weighted aggregation. GraphSAGE~\cite{graphsage} introduces a variety of learnable aggregation functions for performing message aggregation. Moreover, GIN~\cite{xu2018gin} aims to preserve structural information maximally and theoretically proves its discriminative power matches the Weisfeiler-Lehman graph isomorphism test. Further details on GNN research can be found in surveys~\cite{wu2020comprehensive,zhou2020graph}.

\paragraph{Multimodal Federated Learning (MFL).} MFL studies federated optimization over multimodal data. PEPSY~\cite{pepsy} handles both missing modalities and missing input features by learning client-side data-missing profiles for robust aggregation. FedMVP~\cite{fedmvp} leverages frozen pre-trained foundation models for cross-modal completion and representation knowledge transfer, with CKA-based importance-aware aggregation. FedMAC~\cite{fedmac} aligns heterogeneous modality representations via cross-modal contrastive learning with modality-specific adaptive aggregation. These methods address modality missingness at the data level, but they are not designed for graph-structured multimodal data and assume a uniform downstream task across clients.

\section{Additional Experimental Details}
\label{appendix: more experimental setups}

\subsection{Training Protocol}
\label{appendix: training protocol}

Unless otherwise specified, all reported results are presented as mean$\pm$standard deviation over five random seeds. For each seed, all methods share the same data split and, for robustness experiments, the same realization of sparsity patterns. Baseline implementations use the official released code with the recommended hyperparameter settings, while ablated variants inherit the same configuration with only the corresponding component removed. Federated pre-training runs for $T=200$ communication rounds with full client participation and $E=2$ local epochs per round; task-oriented fine-tuning runs for 100 epochs per client.

\paragraph{Perturbation model.}
For the robustness study (Sec.~\ref{sec: robustness}), three types of sparsity are injected before training with ratio $\alpha$, and for a given seed every method receives the identical realization. \textit{Feature sparsity} randomly masks $\alpha$\% of the feature dimensions of every node in the target modality, setting them to zero. \textit{Topology sparsity} removes $\alpha$\% of the existing edges uniformly at random. \textit{Label sparsity} removes $\alpha$\% of the available node labels for fine-tuning, sampled uniformly per client.

\paragraph{Task definitions.}
We consider two graph-centric tasks: (1) \textit{Node Classification}: assigning category labels to unlabeled nodes; (2) \textit{Link Prediction}: estimating whether an edge $(u,v)$ should exist in $\mathcal{E}^k$; and two modality-centric tasks: (1) \textit{Modality Retrieval}: given a query from one modality (e.g., text), retrieving its corresponding representation in another modality; (2) \textit{Modality Generation}: producing content in a target modality conditioned on a target node along with task instructions and its graph neighborhood, covering Graph-to-Text (G2Text) generation of textual descriptions and Graph-to-Image (G2Image) synthesis of visual content.

\subsection{Model Configuration}
\label{appendix: model configuration}

Each node is associated with text and image features extracted from frozen pre-trained encoders (Sentence-BERT~\cite{sentence_bert} or T5~\cite{t5} for text, ViT~\cite{vit} or DINOv2~\cite{dinov2} for images; Qwen2-VL-7B-Instruct~\cite{bai2023qwen} is used for the G2Text/G2Image feature extraction), producing 768-dimensional representations for all datasets. Each modality is encoded by a linear projection (Eq.~\eqref{eq: modality projection}) followed by low-pass graph filtering (Eq.~\eqref{eq: graph filter}).

Unless otherwise stated, the default hyperparameters are shared across all datasets: the latent dimension is $d=128$; the propagation depth is $L=3$; the smoothing strength is $\alpha=2.0$ in Eq.~\eqref{eq: graph filter}; the PageRank damping factor is $0.85$ with 20 power iterations; the contrastive temperature is $\tau=0.5$; the loss weights are $\lambda_m=0.3$, $\lambda_g=0.2$, and $\lambda_c=0.5$, which are normalized to sum to $1$; and the differential-privacy configuration is $\delta=10^{-5}$ with clip bound $C=1.0$ ($\varepsilon=\infty$ unless the privacy budget is explicitly swept in Sec.~\ref{sec: privacy utility}). For the neighbor-level contrastive loss, we simulate 5 random walks of length 10 per node and sample an equal number of negatives. A sensitivity analysis of the main hyperparameters is reported in Sec.~\ref{sec: sensitivity analysis}.

\paragraph{Optimization.}
All models are optimized using Adam with learning rate $10^{-3}$ and weight decay $10^{-5}$ under full-batch training. Encoder parameters are aggregated at the server by node-count-weighted averaging, and no server-side model is maintained for the routing outputs.

\section{Efficiency Analysis}
\label{appendix: more experiments}

We analyze the computational and communication complexity of FedTCR per federated round. On the client side, multimodal graph encoding costs $\mathcal{O}(|\mathcal{M}| \cdot |\mathcal{V}^k| d_m d)$ for the modality projections and $\mathcal{O}(L \cdot |\mathcal{E}^k| d)$ for the low-pass graph filtering with sparse adjacency multiplications. The topology-aware weights are computed via PageRank power iteration at $\mathcal{O}(|\mathcal{E}^k|)$ per iteration, and only once as a preprocessing step before training. Prototype construction costs $\mathcal{O}(|\bar{\mathcal{M}}| \cdot |\mathcal{V}^k| d)$, and the three contrastive losses are computed over per-client sampled node pairs, sharing the same order as standard cross-contrastive multimodal training. On the server side, parameter aggregation costs $\mathcal{O}(|\mathcal{S}| \cdot |\bar{\mathcal{M}}| d_m d)$, while cross-modal routing requires pairwise cosine similarities among $|\mathcal{S}| \times |\bar{\mathcal{M}}|$ prototypes, i.e., $\mathcal{O}(|\mathcal{S}|^2 |\bar{\mathcal{M}}|^2 d)$, which is negligible compared with client-side training since each prototype is a single $d$-dimensional vector. In terms of communication, each client uploads the shared encoder parameters $\{\mathbf{\Phi}^{(m)}\}_{m \in \mathcal{M}}$, identical to standard FL methods such as FedAvg, plus $|\bar{\mathcal{M}}|$ prototype vectors, a marginal overhead of $|\bar{\mathcal{M}}| \cdot d$ values per round. The isolated fine-tuning stage involves no further communication. Overall, FedTCR introduces only lightweight overhead over classical FL methods while enabling cross-task and cross-modal collaboration.

\section{Dataset Details}
\label{appendix: dataset details}

Detailed statistical information on the datasets is presented in Table~\ref{table: datasets}, where the graph statistics follow the OpenMAG construction~\cite{openmag}, with textual descriptions as follows.
\begin{table}[htbp]
\centering
\resizebox{\columnwidth}{!}{
\begin{tabular}{lccccc}
\toprule
Datasets     & \# Modalities & \# Nodes & \# Edges   & \# Classes  & Domain    \\
\midrule
Movies       & Text, Image  & 16,672  & 218,390   & 20               & Movie Network  \\
Grocery      & Text, Image  & 17,074  & 171,340   & 20               & Recommendation \\
RedditS      & Text, Image  & 15,894  & 566,160   & 20               & Social Network \\
BiliDance    & Text, Image  & 2,307   & 9,127     & -                & Entertainment  \\
Toys         & Text, Image  & 20,695  & 126,886   & 18               & Recommendation \\
ele-fashion  & Text, Image  & 97,766  & 199,602   & 12               & Co-purchase Network \\
Flickr30k    & Text, Image  & 31,783  & 181,151   & -                & Image Network  \\
SemArt       & Text, Image  & 21,382  & 1,216,432 & -                & Artwork        \\
\bottomrule
\end{tabular}
}
\caption{Statistics of the experimental datasets.}
\label{table: datasets}
\end{table}

\paragraph{Movies}~\cite{ni2019justifying}
    is sourced from Amazon's Movies and TV category.
    Nodes correspond to DVD/Blu-ray products, and edges reflect consumer co-purchasing behavior.
    Node attributes include textual plot synopses and customer reviews, alongside visual features derived from official cover art.

\paragraph{Grocery}~\cite{ni2019justifying, hou2024bridging}
    is sourced from Amazon's Grocery and Gourmet Food category.
    Nodes correspond to food and household products, and edges reflect co-purchasing behavior.
    Node attributes include textual product descriptions and visual features from product images.
    This dataset is used for node classification, where products are categorized into fine-grained grocery categories.

\paragraph{RedditS}~\cite{RedditS}
    is a social network derived from Reddit, where nodes correspond to posts linked by user interaction and community relations.
    Each node carries textual post content and associated visual context.
    This dataset is utilized for link prediction, forecasting potential interaction links between posts.

\paragraph{BiliDance}~\cite{zhang2024ninerec}
    is a video entertainment network collected from the Bilibili platform, where nodes correspond to dance-video content items linked by user co-viewing behavior.
    Nodes carry textual descriptions and visual cover features.
    This dataset is utilized for link prediction over content-association structures.

\paragraph{Toys}~\cite{ni2019justifying, hou2024bridging}
    is sourced from Amazon's Toys and Games category.
    Nodes correspond to toy products, and edges reflect co-purchasing behavior.
    Node attributes include textual product descriptions and visual features from product images.
    This dataset is utilized for modality retrieval between textual and visual product content.

\paragraph{ele-fashion}~\cite{ni2019justifying, hou2024bridging}
    is a heterogeneous graph merging Amazon's Electronics and Fashion categories. Nodes are connected via cross-category co-purchasing links, revealing latent consumer preferences across disparate domains.
    Features combine technical specs with style descriptions and product imagery.
    This dataset is utilized for modality retrieval between textual and visual product content.

\paragraph{Flickr30k}~\cite{plummer2015_Flickr30k}
    is a canonical image-text reasoning dataset.
    In the OpenMAG setting~\cite{openmag}, we construct a graph where nodes represent image regions and caption phrases, linked by semantic grounding annotations.
    This dataset is utilized for Graph-to-Text (G2Text) generation, evaluating the model's ability to generate descriptive captions by traversing grounded visual-textual relationships.

\paragraph{SemArt}~\cite{garcia2018_SemArt}
    is an artwork understanding dataset of fine-art paintings, where nodes correspond to artworks linked by artistic, historical, and semantic relations.
    Nodes carry visual painting features and textual descriptions of their content and context.
    This dataset is utilized for Graph-to-Image (G2Image) synthesis, producing visual content conditioned on multimodal graph context.

\section{Baseline Details}
\label{appendix: baseline details}

\subsection{Isolated Training}

    \paragraph{MMGCN}~\cite{mmgcn} introduces a multimodal framework for micro-video recommendation by modeling user preferences across visual, acoustic, and textual channels.
    It builds separate modality-specific bipartite graphs, captures high-order interactions within each, and fuses them via a structured integration layer, effectively reflecting user-item dynamics in each sensory modality.
    
    \paragraph{MGAT}~\cite{mgat} applies gated attention over parallel multimodal interaction graphs for personalized recommendation.
    By adaptively weighting different modalities, it disentangles fine-grained user interests and filters out noisy or conflicting signals, enhancing preference modeling robustness.
    
    \paragraph{LGMRec}~\cite{lgmrec} is a multimodal recommender that models both local and global user interests through graph learning. It separates collaborative and multimodal signals in user embeddings and addresses sparsity in local interest modeling. A local graph embedding module learns collaborative and modality-specific embeddings independently, while a global hypergraph module captures overall dependencies among users and items. Combining these decoupled local and global embeddings improves recommendation accuracy and robustness.

\subsection{Classical FL Methods}

\paragraph{FedAvg.}~\cite{fedavg} serves as a foundational method in FL, enabling decentralized model training across diverse devices while preserving data privacy. 
    Initiated by a central server that distributes a global model, clients independently execute local updates through stochastic gradient descent. 
    Subsequently, these updates are aggregated by the server via averaging to refine the global model, with the cycle repeating until convergence.

\paragraph{SCAFFOLD.}~\cite{scaffold} employs control variates to mitigate client-drift in FL. 
    Demonstrating significant reductions in communication rounds, Scaffold is resilient to data heterogeneity and client sampling. 

\paragraph{FedDC.}~\cite{feddc} is a novel FL algorithm that corrects local drift through lightweight modifications. Each client tracks the deviation between local and global model parameters using an auxiliary variable, enhancing parameter-level consistency.

\subsection{Multimodal FL Methods}

\paragraph{PEPSY}~\cite{pepsy} addresses multimodal federated learning with both missing modalities and missing input features by learning client-side data-missing profiles that encode local missing patterns as embedding controls. These profiles are probabilistically aligned and aggregated on the server to reconfigure shared representations toward each client’s incomplete data view, enabling robust aggregation and stable performance under severe data incompleteness.

\paragraph{FedMVP}~\cite{fedmvp} addresses modality missing in multimodal federated learning by leveraging frozen pre-trained foundation models for cross-modal completion and representation knowledge transfer. It trains a lightweight joint encoder via multimodal contrastive objectives and performs CKA-based importance-aware aggregation on the server, achieving robust performance under severe modality incompleteness. 

\paragraph{FedMAC}~\cite{fedmac} addresses cross-modal heterogeneity in multimodal federated learning by introducing a modality-aware collaboration framework that aligns heterogeneous modality representations via cross-modal contrastive learning. It further employs modality-specific aggregation with adaptive weighting to mitigate negative transfer across clients with inconsistent modality availability.

\subsection{Unimodal FGL Methods}

\paragraph{FedSage+}~\cite{fedsage_plus} extends FedSage to the subgraph federated learning setting by explicitly addressing missing cross-client neighbors. It jointly trains a GraphSAGE classifier with a local missing-neighbor generator that synthesizes potential cross-subgraph neighbors, enabling more complete neighborhood aggregation under federation and improving global generalization without sharing raw graph data.

\paragraph{FedGTA}~\cite{fedgta} innovatively merges large-scale graph learning with federated graph learning. Clients encode topology and node attributes, compute local smoothing confidence and mixed moments of neighbor features, and then upload these to the server. The server uses this data to perform personalized model aggregation, utilizing local smoothing confidence as weights for effective integration.

\paragraph{FedSPA}~\cite{fedspa} addresses homophily heterogeneity in federated graph learning by explicitly modeling both homophily conflict and homophily bias across clients. It introduces Subgraph Feature Propagation Decoupling (SFPD) to separate homophilic and heterophilic message passing, enabling collaboration under unified homophily levels, and proposes Homophily Bias-Driven Aggregation (HBDA) to adaptively weight client contributions based on spectral and parameter-sensitivity cues, thereby improving global generalization.

\paragraph{FedIIH}~\cite{fediih} addresses heterogeneity in federated graph learning by jointly modeling inter-client and intra-client heterogeneity. It infers subgraph distribution similarities via a hierarchical variational framework from a global perspective, while disentangling local subgraphs into multiple latent factors to enable factor-wise personalized federation, leading to robust collaboration across both homophilic and heterophilic graphs.

\paragraph{FedSSP}~\cite{fedssp} proposes a personalized federated graph learning framework that addresses cross-domain structural heterogeneity from a spectral perspective. It shares generic spectral knowledge across clients to mitigate knowledge conflict induced by domain shifts, while retaining client-specific components to preserve personalization.

\paragraph{S2FGL}~\cite{s2fgl} tackles subgraph federated graph learning by jointly addressing spatial label-signal disruption and spectral client drift. It reinforces missing label semantics via prototype-based semantic sharing and aligns graph-frequency components across clients to improve robustness and generalization under heterogeneous subgraph distributions.

\paragraph{FedLap}~\cite{fedlap} leverages global graph structure in subgraph federated learning via Laplacian smoothing, introducing a structural regularization term that implicitly enforces representation consistency among neighboring nodes without explicit message passing or feature sharing. By operating in the spectral domain, FedLap captures inter-subgraph dependencies while achieving strong privacy guarantees and low communication overhead.

\section{Theoretical Proofs}
\label{appendix: proofs}

\begin{theorem}[Unified view of propagation as graph filtering]
\label{thm: smooth fused rep}
Let $\tilde{\mathbf{A}} \in \mathbb{R}^{n \times n}$ be the symmetrically normalized adjacency matrix of an undirected graph and $\mathbf{Z} \in \mathbb{R}^{n \times d}$ be node representations. Consider the general depth-$L$ polynomial graph filter, which can be formulated as follows:
\begin{equation}
\label{eq: general filter}
\mathbf{H} = \Big(\sum_{l=0}^{L} \gamma_l \tilde{\mathbf{A}}^{l}\Big) \mathbf{Z}, \qquad \gamma_l \ge 0,\ \sum_{l=0}^{L} \gamma_l = 1.
\end{equation}
Then: (i) SGC and the linearized $L$-layer GCN correspond to the one-hot choice $\gamma_L = 1$; (ii) APPNP-style personalized-PageRank propagation corresponds to the geometric choice $\gamma_l = (1-\beta)\,\beta^{l}$ with $\beta \in (0,1)$ and $L \to \infty$; (iii) the FedTCR filter of Eq.~\eqref{eq: graph filter} is exactly the geometric choice with $\beta = \alpha/(\alpha+1)$, i.e., the finite-$L$ truncation of (ii), whose approximation error decays exponentially:
\begin{equation}
\label{eq: truncation bound}
\big\|\mathbf{H}_L - \mathbf{H}_\infty\big\|_F \;\le\; \beta^{L+1}\,\|\mathbf{Z}\|_F .
\end{equation}
\end{theorem}

\begin{proof}
Since $\tilde{\mathbf{A}}$ is symmetric with spectrum in $[-1, 1]$, we have $\|\tilde{\mathbf{A}}\|_2 \le 1$, so for any $\beta \in (0,1)$ the Neumann series $\sum_{l=0}^{\infty} (\beta \tilde{\mathbf{A}})^l$ converges to $(\mathbf{I} - \beta \tilde{\mathbf{A}})^{-1}$.

(i) With $\gamma_L = 1$ and all other coefficients zero, Eq.~\eqref{eq: general filter} reduces to $\mathbf{H} = \tilde{\mathbf{A}}^{L} \mathbf{Z}$, which is precisely the propagation operator of SGC; an $L$-layer GCN without intermediate nonlinearities and with shared collapsed weights realizes the same operator as a special case of the general filter.

(ii) With $\gamma_l = (1-\beta)\beta^{l}$ and $L \to \infty$, the coefficient sequence sums to one, and
$\sum_{l=0}^{\infty} (1-\beta)(\beta \tilde{\mathbf{A}})^l = (1-\beta)(\mathbf{I} - \beta \tilde{\mathbf{A}})^{-1}$,
which is the personalized-PageRank propagation matrix with teleport probability $1-\beta$, i.e., the exact solution that APPNP approximates with a finite number of power-iteration steps as its propagation matrix.

(iii) Setting $\beta = \alpha/(\alpha+1)$ gives $1 - \beta = 1/(\alpha+1)$, so the coefficients of Eq.~\eqref{eq: graph filter} satisfy $\gamma_l = \frac{1}{\alpha+1}\big(\frac{\alpha}{\alpha+1}\big)^{l} = (1-\beta)\beta^{l}$, i.e., FedTCR's filter is the finite-$L$ truncation of (ii). For the error, writing $\mathbf{H}_\infty = (1-\beta)\sum_{l=0}^{\infty}(\beta\tilde{\mathbf{A}})^l \mathbf{Z}$,
\begin{equation*}
\begin{aligned}
\|\mathbf{H}_L - \mathbf{H}_\infty\|_F
&\le (1-\beta) \sum_{l=L+1}^{\infty} \beta^{l} \|\tilde{\mathbf{A}}\|_2^{l}\, \|\mathbf{Z}\|_F \\
&\le (1-\beta)\,\frac{\beta^{L+1}}{1-\beta}\,\|\mathbf{Z}\|_F
= \beta^{L+1} \|\mathbf{Z}\|_F ,
\end{aligned}
\end{equation*}
where the first inequality applies submultiplicativity of the Frobenius norm and the second uses $\|\tilde{\mathbf{A}}\|_2 \le 1$.
\end{proof}

\section{Metric Details}
\label{appendix: metric details}

All metrics are computed on each client's held-out split after isolated task-oriented fine-tuning, and are reported as mean$\pm$standard deviation over five random seeds (Appendix~\ref{appendix: training protocol}). Task definitions follow Appendix~\ref{appendix: more experimental setups}.

\paragraph{Node Classification.} We report classification accuracy (ACC):
\begin{equation}
    \text{ACC} = \frac{1}{N} \sum_{i=1}^{N} \mathbb{I}(\hat{y}_i = y_i),
\end{equation}
where $N$ is the number of test nodes, $y_i$ the ground-truth label, $\hat{y}_i$ the prediction, and $\mathbb{I}(\cdot)$ the indicator function.

\paragraph{Link Prediction.} We report the Area Under the ROC Curve (AUC), which measures the probability that a randomly chosen positive (existent) edge receives a higher score than a randomly chosen negative (non-existent) node pair:
\begin{equation}
    \text{AUC} = P\big(s(u, v) > s(u', v')\big),
\end{equation}
where $s(\cdot,\cdot)$ is the model's similarity score, $(u,v)$ is a positive edge, and $(u',v')$ is a sampled negative pair.

\paragraph{Modality Retrieval.} We report Recall@1 (R@1), the fraction of queries whose ground-truth cross-modal target is ranked first:
\begin{equation}
    \text{R@1} = \frac{1}{|Q|} \sum_{i=1}^{|Q|} \mathbb{I}(\text{rank}_i = 1),
\end{equation}
where $Q$ is the query set and $\text{rank}_i$ is the rank of the correct target for query $i$.

\paragraph{Graph-to-Text (G2Text).} Following MMGL~\cite{yoon2023mmgl}, we report ROUGE-L (R-L), which measures sentence-level recall based on the longest common subsequence between the generated description and the reference:
\begin{equation}
    \text{ROUGE-L} = \frac{(1 + \beta^2) R_{lcs} P_{lcs}}{R_{lcs} + \beta^2 P_{lcs}}.
\end{equation}

\paragraph{Graph-to-Image (G2Image).} Following InstructG2I~\cite{instructg2i}, we report CLIP-Score (C-S), which quantifies cross-modal semantic consistency between the generated image and its conditioning description via pre-trained CLIP encoders ($E_T$ for text and $E_I$ for images):
\begin{equation}
    \text{CLIP-Score}(I, T) = \max\left(100 \cdot \cos(E_I(I), E_T(T)), 0\right).
\end{equation}

\section{Limitations}
\label{appendix: lim}

 While FedTCR demonstrates strong empirical performance across diverse FMGL scenarios, several limitations warrant discussion. First, each client distills its modality-specific knowledge into a single prototype vector per modality, which may not fully capture the internal multi-community structure of large-scale MAGs with diverse semantic clusters. Extending to multiple prototypes per modality (e.g., cluster-aware prototypes) could improve expressiveness at the cost of increased communication overhead. Second, our evaluation focuses on MAGs with text and image modalities; the generalizability to other modality types (e.g., audio, video, tabular data) remains to be validated. Third, the two-stage paradigm assumes that the pre-training phase has sufficiently converged before fine-tuning begins, and the potential gap between the task-agnostic objective and specific downstream tasks may limit the adaptability for highly specialized applications. Exploring progressive or continual pre-training strategies could mitigate this issue.

\section{Environment}
\label{appendix: environment}

All experiments are conducted on a workstation equipped with Intel Xeon Scalable processors and NVIDIA RTX 6000 Ada Generation GPUs with 96 GB of VRAM, supported by 256 GB of system RAM. The computational environment utilizes CUDA 12.9, while software implementations are developed using Python 3.10.18 and PyTorch 2.8.

\section{Pseudocode}
\label{appendix: pseudocode}

The complete procedure of FedTCR is presented in Algorithm~\ref{alg: fedtcr}, which consists of two stages: federated task-agnostic pre-training and isolated task-oriented fine-tuning.

\normalem
\begin{algorithm*}[t]
\caption{FedTCR: FMGL with Topology-aware Cross-modal Routing}
\label{alg: fedtcr}
\KwInput{$K$ clients with MAGs $\{\mathcal{G}^k\}_{k=1}^K$; modality set $\mathcal{M}$; extended modality set $\bar{\mathcal{M}} = \mathcal{M} \cup \{|\mathcal{M}|+1\}$; communication rounds $T$; local epochs $E$; smoothing strength $\alpha$; propagation depth $L$; temperature $\tau$; loss weights $\lambda_m, \lambda_g, \lambda_c$}
\KwOutput{Task-specific predictions for each client}
\tcc{Stage 1: Federated Task-agnostic Pre-training}
Initialize modality-specific encoders $\{\mathbf{\Phi}^{(m)}\}_{m \in \mathcal{M}}$\;
\For{ each client $k \in [K]$ \textbf{in parallel}}{
    Compute PageRank importance weights $\{\omega_i^k\}_{i \in \mathcal{V}^k}$ via power iteration (Eq.~\eqref{eq: node importance})\;
    Simulate random walks to obtain positive neighbor sets $\{\mathcal{P}_i^k\}_{i \in \mathcal{V}^k}$ and sample negative sets $\{\mathcal{N}_i^k\}_{i \in \mathcal{V}^k}$\;
}
\For{each round $t = 1, 2, \ldots, T$}{
    Server distributes global encoders $\{\mathbf{\Phi}^{(m)}\}_{m \in \mathcal{M}}$ to all clients\;
    \lIf{$t \geq 2$}{Server distributes routed prototypes $\{(\mathbf{P}_{\mathrm{pos}}^{k,(i,j)}, \mathbf{P}_{\mathrm{neg}}^{k,(i,j)})\}_{i,j}$ to each client $k$}
    \tcc{Client-side Local Training}
    \For{ each client $k \in [K]$ \textbf{in parallel}}{
        \For{each local epoch $e = 1, 2, \ldots, E$}{
            \tcc{Multimodal Graph Encoding (Eq.~\eqref{eq: modality projection}, \eqref{eq: graph filter})}
            $\mathbf{Z}^{k,(m)} \leftarrow \mathrm{L2Norm}(\mathbf{X}^{k,(m)} \mathbf{\Phi}^{(m)})$, $\forall m \in \mathcal{M}$\;
            $\bar{\mathbf{Z}}^k \leftarrow \frac{1}{|\mathcal{M}|}\sum_{m \in \mathcal{M}} \mathbf{Z}^{k,(m)}$\;
            $\mathbf{H}^k \leftarrow \frac{1}{\alpha+1}\sum_{l=0}^{L}\left(\frac{\alpha}{\alpha+1}\tilde{\mathbf{A}}^k\right)^l \bar{\mathbf{Z}}^k$\;
            \tcc{Topology-aware Prototype Construction (Eq.~\eqref{eq: prototype construction})}
            $\mathbf{P}^{k,(m)} \leftarrow \sum_{i \in \mathcal{V}^k} \omega_i^k \cdot \mathbf{z}_i^{k,(m)}$, $\forall m \in \bar{\mathcal{M}}$\;
            \tcc{Local Optimization (Eq.~\eqref{eq: modal alignment}, \eqref{eq: graph contrastive}, \eqref{eq: cross client loss})}
            Compute node-level contrastive loss $\mathcal{L}_m$\;
            Compute neighbor-level contrastive loss $\mathcal{L}_g$\;
            \lIf{$t \geq 2$}{Compute client-level contrastive loss $\mathcal{L}_c$ using routed prototypes}
            Update $\{\mathbf{\Phi}^{(m)}\}_{m \in \mathcal{M}}$ via $\nabla(\lambda_m \mathcal{L}_m + \lambda_g \mathcal{L}_g + \lambda_c \mathcal{L}_c)$ (Eq.~\eqref{eq: total loss})\;
        }
        Upload updated encoders $\{\mathbf{\Phi}^{k,(m)}\}_{m \in \mathcal{M}}$ and final prototypes $\{\mathbf{P}^{k,(m)}\}_{m \in \bar{\mathcal{M}}}$ to server\;
    }
    \tcc{Server-side Aggregation and Routing}
    Aggregate encoders: $\mathbf{\Phi}^{(m)} \leftarrow \sum_{k=1}^{K} \frac{|\mathcal{V}^k|}{N} \mathbf{\Phi}^{k,(m)}$, $\forall m \in \mathcal{M}$\;
    \tcc{Cross-modal Prototype Routing (Eq.~\eqref{eq: prototype routing})}
    \For{ each client $a \in [K]$}{
        \For{ each modality pair $(i, j),\; i \in \bar{\mathcal{M}},\; j \in \bar{\mathcal{M}} \setminus \{i\}$}{
            $\mathbf{P}_{\mathrm{pos}}^{a,(i,j)} \leftarrow \hat{\mathbf{P}}^{b^*,(j)}$, where $b^* = \argmax_{b \neq a} \cos(\hat{\mathbf{P}}^{a,(i)}, \hat{\mathbf{P}}^{b,(j)})$\;
            $\mathbf{P}_{\mathrm{neg}}^{a,(i,j)} \leftarrow \hat{\mathbf{P}}^{b',(j)}$, where $b' = \argmin_{b \neq a} \cos(\hat{\mathbf{P}}^{a,(i)}, \hat{\mathbf{P}}^{b,(j)})$\;
        }
    }
}
\tcc{Stage 2: Isolated Task-oriented Fine-tuning}
\For{ each client $k \in [K]$ \textbf{in parallel}}{
    Attach task-specific head $g^k$ to the pre-trained encoder\;
    Fine-tune $\{\mathbf{\Phi}^{(m)}\}_{m \in \mathcal{M}}$ and $g^k$ on local MAG $\mathcal{G}^k$ with task-specific loss $\mathcal{L}_{\mathrm{task}}$\;
}
\end{algorithm*}
\ULforem

\end{document}